\documentclass[preprint,12pt]{elsarticle}

\usepackage{graphicx}
\usepackage{amsmath, amssymb, amsfonts, amsthm}
\usepackage[ruled,vlined]{algorithm2e}
\usepackage{caption}
\usepackage{subcaption}
\usepackage{cleveref}
\usepackage{xfrac}
\usepackage[normalem]{ulem} 					% for editing purposes 
\usepackage{color} 								% for editing purposes 
\usepackage{lineno}
\usepackage{url}
  \newtheorem{theorem}{Theorem}

  \newtheorem{corollary}{Corollary}

\newcommand{\pr}{\ensuremath{\mathrm{Pr}}}

\newcommand{\tw}{\ensuremath{\mathrm{tw}}}
\newcommand{\bn}{\ensuremath{\mathcal{B}}}

\newcommand{\true}{\ensuremath{T}}
\newcommand{\false}{\ensuremath{F}}
\newcommand{\yes}{{\em yes}}
\newcommand{\no}{{\em no}}

\newcommand{\mprob}[1] {\pr(#1)}
\newcommand{\cprob}[2] {\pr(#1\;|\;#2)}

\newcommand{\dproblem}[3]{
\begin{quote}
\begin{flushleft} 
  \noindent 
  \textsc{#1}\\
  \textbf{Instance: }#2.\\
  \textbf{Question: }#3?\\
\end{flushleft}
\end{quote}}

\newcommand{\fproblem}[3]{
\begin{quote}
\begin{flushleft} 
  \noindent 
  \textsc{#1}\\
  \textbf{Instance: }#2.\\
  \textbf{Output: }#3.\\
\end{flushleft}
\end{quote}}

\newcommand{\pt}{\ensuremath{\mathsf{P}}}
\newcommand{\np}{\ensuremath{\mathsf{NP}}}
\newcommand{\pp}{\ensuremath{\mathsf{PP}}}
\newcommand{\bpp}{\ensuremath{\mathsf{BPP}}}
\newcommand{\nppp}{\ensuremath{\mathsf{NP^{\mathsf{PP}}}}}

\newcommand{\conppp}{\ensuremath{\textrm{co-}\mathsf{NP^{\mathsf{PP}}}}}
\newcommand{\npnppp}{\ensuremath{\mathsf{NP^{\mathsf{NP^{\mathsf{PP}}}}}}}

\newcommand{\map}{\textsc{MAP}}

\newcommand{\maximin}{\textsc{Maximin A Posteriori}}
\newcommand{\monotone}{\textsc{Monotonicity}}
\newcommand{\tuning}{\textsc{Parameter Tuning}}
\newcommand{\tunablemonotone}{\textsc{Tunable Monotonicity}}
\newcommand{\mindep}{\textsc{MAP-independence}}
\newcommand{\wmindep}{\textsc{Weak MAP-independence}}
\newcommand{\maxmindep}{\textsc{Maximum MAP-independence}}
\newcommand{\mdep}{\textsc{MAP-dependence}}

\newcommand{\infer}{\textsc{Inference}}

\newcommand{\majsat}{\textsc{Majsat}}
\newcommand{\amajsat}{\textsc{A-Majsat}}

\newcommand{\eamajsat}{\textsc{EA-Majsat}}
\newcommand{\partamajsat}{\textsc{Partition-free A-Majsat}}

\journal{International Journal of Approximate Reasoning}

\begin{document}

\begin{frontmatter}

\title{Motivating explanations in Bayesian networks using MAP-independence}
\author[inst1]{Johan Kwisthout}
\ead{j.kwisthout@donders.ru.nl}
\ead[url]{http://www.socsci.ru.nl/johank/}

\affiliation[inst1]{
            organization={Donders Institute for Brain, Cognition, and Behaviour, Radboud University},
            addressline={\\Thomas van Aquinostraat 4}, 
            city={Nijmegen},
            postcode={6525GD}, 
            country={The Netherlands}}

\begin{abstract}
In decision support systems the motivation and justification of the system's diagnosis or classification is crucial for the acceptance of the system by the human user. In Bayesian networks a diagnosis or classification is typically formalized as the computation of the most probable joint value assignment to the hypothesis variables, given the observed values of the evidence variables (generally known as the MAP problem). While solving the MAP problem gives the most probable explanation of the evidence, the computation is a black box as far as the human user is concerned and it does not give additional insights that allow the user to appreciate and accept the decision. For example, a user might want to know to whether an unobserved variable could potentially (upon observation) impact the explanation, or whether it is irrelevant in this aspect. In this paper we introduce a new concept, MAP-independence, which tries to capture this notion of relevance, and explore its role towards a potential justification of an inference to the best explanation. We formalize several computational problems based on this concept and assess their computational complexity.
\end{abstract}

\begin{highlights}
\item We introduce MAP-independence as a novel concept in Bayesian networks, indicating potential impact of an intermediate (hidden) variable to the MAP explanation.
\item We discuss how this concept may contribute to justifying MAP explanations, for example in the context of a decision support system.
\end{highlights}


\end{frontmatter}

%\linenumbers

\section{Introduction}

With the availability of petabytes of data and the emergence of `deep' learning as an AI technique to find statistical regularities in these large quantities of data, artificial intelligence in general and machine learning in particular has arguably entered a new phase since its emergence in the 1950s. Deep learning aims to build hierarchical models representing the data, with every new layer in the hierarchy representing ever more abstract information; for example, from individual pixels to lines and curves, to geometric patterns, to features, to categories. Superficially this might be related to how the human visual cortex interprets visual stimuli and seeks to classify a picture to be that of a cat, rather than of a dog.

When describing in what sense a cat is different from a dog, humans may use features and categories that we agreed upon to be defining features of cats and dogs, such as whiskers, location and form of the ears, the nose, etc. The deep learning method, however, does not adhere to features we humans find to be good descriptors; it bases its decisions where and how to `carve nature's joints' solely on basis of the statistics of the data. Hence, it might very well be that the curvature of the spine (or some other apparently `random' feature) happens to be {\em the} statistically most important factor to distinguish cats from dogs. This imposes huge challenges when the machine learning algorithm is asked to {\em justify} its classification to a human user. The sub-field of {\em explainable AI} has recently emerged to study how to align statistical machine learning with informative user-based motivations or explanations. 
Explainable AI, however, is not limited to deep neural network applications. Any AI application where trustworthiness is important benefits from justification and transparency of its internal process \cite{Gunning19}, and this includes decision support systems that are based on Bayesian networks \cite{Mihaljevic21}, which is the focus of this paper. In these systems typically one is interested in the hypothesis that best explains the available evidence; for example in a medical case, the infection that is most probable given a set of health complaints and test findings \cite{Kyrimi21}.

Note that `explainability' in explainable AI is in principle a triadic relationship between what needs to be explained, the explanation, and the {\em user who seeks the explanation} \cite{Ras18}. An explanation will be more satisfying (`lovelier', in Peter Lipton's \cite{Lipton91} terms) if it allows the user to gain more understanding about the phenomenon to be explained. Lacave and D\'{i}ez \cite{Lacave00} review explanation in the context of a Bayesian network. In their work, three major categories of explanation has been identified. For example, the focus can be on explanation of the {\em evidence} or observed variables that need to be explained, where the most probable joint value assignment (or {\em MAP explanation}) serves as this explanation\footnote{In this paper we do not touch the question whether `best' is to be identified with `most probable'. The interested reader is referred to the vast literature on inference to the best explanation such as \cite{Lipton91}, and more in particular to some of our earlier work \cite{Kwisthout13} that discusses the trade-off between probability and informativeness of explanations.}. Alternatively, the focus can be on explanation of the {\em model}, i.e., the structure of the Bayesian network can be explained by providing conditional independence tests and justification in the form of expert assessment of probability distributions or training data sets. Finally, and for us most relevant, the focus can be on explaining the reasoning process. Here, the MAP explanation itself needs to be explained or motivated in terms of an explication of the reasoning process that needs further justification.

In this paper our aim is to improve the user's understanding of a specific MAP explanation by {\em explicating the relevant information} that contributed to said MAP explanation. That is, we seek to {\em justify} the MAP explanation by zooming in on the reasoning process. In some way, in the computation of the most probable joint value assignment to the hypothesis variables given the evidence, the process of marginalizing out the intermediate variables makes the decision more opaque. Some of these variables may have a bigger impact (i.e., are more influential) on the eventual decision than others, and this information is lost in the process. For example, the absence of a specific test result (i.e., a variable we marginalize out in the MAP computation) may lead to a different explanation of the available evidence compared to when a negative (or positive) test result {\em were} present. In this situation, this variable is more influential to the eventual explanation, and hence more relevant for the justification, than if the best explanation would be the same, irrespective of whether the test result was positive, negative, or missing. Our approach in this paper is to justify an explanation by showing which of these variables were influential towards arriving at this explanation. An intermediate variable is {\em relevant} in this justification if it has potential {\em influence} on the explanation.

This approach complements other approaches that focus on relevance of variables for justifying decisions. For MAP (i.e., where a Bayesian network is partitioned into hypothesis, evidence, and intermediate nodes) several authors focus on minimizing the set of hypothesis variables, pruning those variables that are deemed less relevant, thereby avoiding over-specification \cite{Shimony91,Yuan11}; see \cite{Kwisthout13} for an alternative proposal where over-specification is avoided by allowing the best explanation to be a {\em set} of joint value assignments rather than a singleton set. Others propose an assessment of the relevance of the evidence variables (rather than intermediate variables) for justifying an explanation \cite{Suermondt92,Meekes15,Kyrimi16}. Similarly, in Bayesian classifiers (without intermediate nodes) justification is often sought in terms of the impact an observable variable may have on the classification \cite{Albini21,Ribeiro16,Lundberg17}. In contrast to these approaches, we assume that the partition between hypothesis variables, evidence variables, and intermediate variables is a given, and we focus on assessing the contribution of the intermediate variables towards a specific explanation.

Our perspective has roots in Pearl's early work on conditional independence \cite{Pearl87}. Pearl suggests that human reasoning is in principle based on {\em conditional independence}: The  organizational structure of human memory is such that it allows for easily retrieving context-dependent relevant information. For example (from \cite[p.3]{Pearl87}): The color of my friend's car is normally not related to the color of my neighbour's car. However, when my friend tells me she almost mistook my neighbour's car for her own, this information suddenly becomes relevant for me to understand, and for her to explain, this mistake. That is, the color of both cars is independent but becomes conditionally dependent on the evidence\footnote{Graphically one can see this as a so-called common-effect structure, where $C_1$ and $C_2$ are variables that represent my car's, respectively my neighbour's car's, color; both variables have a directed edge towards the variable $M$ that indicates whether my friend misidentified the cars or not. When $M$ is unobserved, $C_1$  and $C_2$ are independent, but they become conditionally dependent on observation of $M$.}. 

In this paper we will argue that Pearl's proposal to model context-dependent (ir)relevance as conditional (in)dependence is in fact too strict. It generally leads to too many variables that are considered to be relevant: for some it is likely the case that, while they may not be conditionally independent on the hypothesized explanation given the evidence, they do not contribute to understanding {\em why} some explanation $h$ is better than the alternatives. That means, for justification purposes their role is limited. In the remainder of this paper we will build on Pearl's work, yet provide a stronger notion of context-dependent relevance and irrelevance of variables relative to explanations of observations. Our goal is to advance explainable AI in the context of Bayesian networks by formalizing the problem of {\em justification} of an explanation (i.e., given an AI-generated explanation, advance the user's understanding why this explanation is preferred over others) into a computational problem that captures some aspects of this justification; in particular, by opening up the `marginalization black box' and show which variables contributed to this decision. We show that this problem is intractable in the general case, but also give fixed-parameter tractability results that show what constraints are needed to render it tractable.

To summarize, we are interested in the potential applicability of this new concept for motivation and justification of MAP explanations, with a focus on its theoretical properties. The remainder of this paper is structured as follows. In the next section we offer some preliminary background on Bayesian networks and computational complexity and share our conventions with respect to notation with the reader. In section \ref{sec:MapIndep} we introduce so-called MAP-independence as an alternative to conditional independence, embed this concept in the literature, and elaborate on the potential of these computational problems for justifying explanations in Bayesian networks. In section \ref{sec:Formal} we introduce several formal computational problems based on this notion, and give complexity proofs and fixed-parameter tractability results for these problems. We conclude in section \ref{sec:Conclusion}.

\section{Preliminaries and notation}

In this section we give some preliminaries and introduce the notational conventions we use throughout this paper. The reader is referred to textbooks like \cite{Darwiche09} for more background both on Bayesian networks and relevant aspects of computational complexity theory.

A Bayesian network $\bn = (\mathbf{G}_{\mathcal{B}}, \pr)$ is a probabilistic graphical model that succinctly represents a joint probability distribution $\mprob{\mathbf{V}} = \prod_{i=1}^n \cprob{V_i}{\pi(V_i)}$ over a set of discrete random variables $\mathbf{V}$. $\bn$ is defined by a directed acyclic graph $\mathbf{G}_{\mathcal{B}} = (\mathbf{V}, \mathbf{A})$, where $\mathbf{V}$ represents the stochastic variables and $\mathbf{A}$ models the conditional (in)dependencies between them, and a set of parameter probabilities $\pr$ in the form of conditional probability tables (CPTs). In our notation $\pi(V_i)$ denotes the set of parents of a node $V_i$ in $\mathbf{G}_{\mathcal{B}}$. We use upper case to indicate variables, lower case to indicate a specific value of a variable, and boldface to indicate sets of variables respectively joint value assignments to such a set. $\Omega(V_i)$ denotes the set of value assignments to $V_i$, with $\Omega(\mathbf{V_a})$ denoting the set of joint value assignment to the set $\mathbf{V_a}$.

One of the key computational problems in Bayesian networks is the problem to find the most probable explanation for a set of observations, i.e., a joint value assignment to a designated set of variables (the explanation set) that has maximum posterior probability given the observed variables (the joint value assignment to the evidence set) in the network. If the network includes variables that are neither observed nor to be explained (known as intermediate variables) this problem is typically referred to as \map. We use the following formal definition:

\fproblem{\map} 
{A Bayesian network $\bn = (\mathbf{G}_{\mathcal{B}}, \pr)$, where $\mathbf{V}(\mathbf{G}_{\mathcal{B}})$ is partitioned into a set of evidence nodes $\mathbf{E}$ with a joint value assignment $\mathbf{e}$, a set of intermediate nodes $\mathbf{I}$, and an explanation set $\mathbf{H}$}
{A joint value assignment $\mathbf{h^*}$ to $\mathbf{H}$ such that for all joint value assignments $\mathbf{h^{\prime}}$ to $\mathbf{H}$, $\cprob{\mathbf{h^*}}{\mathbf{e}} \geq \cprob{\mathbf{h^{\prime}}}{\mathbf{e}}$}

We assume that the reader is familiar with standard notions in computational complexity theory, notably the classes \pt\ and \np, \np-hardness, and polynomial time (many-one) reductions. The class \pp\ is the class of decision problems that can be decided by a probabilistic Turing machine in polynomial time; that is, where \yes-instances are accepted with probability strictly larger than $\sfrac{1}{2}$ and \no-instances are accepted with probability no more than $\sfrac{1}{2}$. A problem in \pp\ might be accepted with probability $\sfrac{1}{2} + \epsilon$ where $\epsilon$ may depend exponentially on the input size $n$. Hence, it may take exponential time to increase the probability of acceptance (by repetition of the computation and taking a majority decision) close to $1$; this in contrast to the related complexity class \bpp\ that contains problems that can be solved in polynomial time by a randomized algorithm with probability arbitrarily close to $1$. This makes \pp\ a powerful class; we know for example that $\np\subseteq\pp$ and the inclusion is assumed to be strict. The canonical \pp-complete decision problem is \majsat: given a Boolean formula $\phi$, does the majority of truth assignments to its variables satisfy $\phi$?

In computational complexity theory, so-called {\em oracles} are theoretical constructs that increase the power of a specific Turing machine. An oracle (e.g., an oracle for \pp-complete problems) can be seen as a `magic sub-routine' that answers class membership queries (e.g, in \pp) in a single time step. In this paper we are specifically interested in classes defined by non-deterministic Turing machines with access to a \pp-oracle. Such a machine is very powerful, and likewise problems that are complete for the corresponding complexity classes \nppp\ (such as \map\ and \tuning) and \conppp\ (such as \monotone) are highly intractable \cite{KwisthoutGaag08,Park04,Gaag04}. Machines that can be characterized by such classes can themselves also be used as an oracle machine to, e.g., a non-deterministic Turing machine, leading to classes like \npnppp\ with complete problems such as \maximin\ and \tunablemonotone\ \cite{Campos05,Kwisthout09}.

Finally we wish to introduce the concept of {\em fixed-parameter tractability} to the reader. Oftentimes, \np-hard computational problems can be rendered tractable when the set of instances is constrained to instances where the value of some (set of) problem parameter(s) is bounded. Formally a parameterized problem $k-\Pi$ is called fixed-parameter tractable for a parameter $k$ if an instance $i$ of $\Pi$ can be decided in time $\mathcal{O}(p(|i|)f(k))$ for a polynomial $p$ and an arbitrary function $f$. Following the discussion in \cite{Kwisthout11} we liberally mix integer and (monotone) rational parameters in our analyses.

\section{MAP-independence}
\label{sec:MapIndep}

In the previous section we introduced the MAP problem as the problem to find the most probable explanation (viz., the joint value assignment to a set of hypothesis variables given observations in the network). In this paper we are not so much concerned with {\em finding} this explanation, but rather with {\em motivating} what information did or did not contribute to a given explanation. That is, rather than providing the `trivial' justification ``$\mathbf{h^*}$ is the best explanation for $\mathbf{e}$, since $\mathrm{argmax}_{\mathbf{h}} \cprob{\mathbf{H}=\mathbf{h}}{\mathbf{e}} = \mathrm{argmax}_{\mathbf{h}} \sum_{\mathbf{i} \in \Omega(\mathbf{I})} \cprob{\mathbf{H}=\mathbf{h},\mathbf{i}}{\mathbf{e}} = \mathbf{h^*}$'' our goal is to partition the set $\mathbf{I}$ into variables $\mathbf{I^+}$ that are {\em relevant} to establishing the best explanation and variables that are {\em irrelevant}. One straightforward approach, motivated by Pearl \cite{Pearl87}, would be to include variables in $\mathbf{I^+}$ if they are conditionally dependent on $\mathbf{H}$ given $\mathbf{e}$, and in $\mathbf{I^-}$ when they are conditionally independent from $\mathbf{H}$ given $\mathbf{e}$ and to {\em motivate} the sets $\mathbf{I^+}$ and $\mathbf{I^-}$ in terms of a set of independence relations. This is particularly useful when inclusion in $\mathbf{I^+}$ is {\em triggered} by the presence of an observation, such as in Pearl's example where `color of my friend's car' and `color of my neighbour's car' become dependent on each other once we learn that my friend confused both cars.

We argue that this way of partitioning intermediate variables into relevant and irrelevant ones, however useful, might not be the full story with respect to explanation. There is a sense in which a variable has an explanatory role in motivating a conclusion that goes beyond conditional (in)dependence. Take for example the small binary network in Figure \ref{fig:xai_inter}\subref{fig:xai_example}. Assume that we want to motivate the best explanation for $A$ given the evidence $C = c$, i.e., we want to motivate the outcome of $\mathrm{argmax}_a \cprob{A=a}{C=c} = \mathrm{argmax}_a \sum_{B,D} \cprob{A=a,B,D}{C=c}$ in terms of variables that contribute to this explanation. Now, obviously $D$ is {\em not} relevant, as it is d-separated from $A$ given $C$. But the role of $B$ is less obvious. This node is obviously not conditionally independent from $A$ given $C$. 

Whether $B$ plays a  role in motivating the outcome of the MAP query is dependent on whether $\mathrm{argmax}_a \sum_{D} \cprob{A=a,B = b,D}{C=c} = \mathrm{argmax}_a \sum_{D} \cprob{A=a,B = \bar b,D}{C=c}$. If both are equal ($B$'s value, were it observed, would have been irrelevant to the MAP query) then $B$ arguably has no explanatory role. If both are unequal than the fact that $B$ is unobserved may in fact be crucial for the explanation. For example, if $B$ represents a variable that encodes a `default' ($b$) versus `fault' ($\bar b$) condition (with $\mprob{b} > \mprob{\bar b}$) the {\em absence} of information about a possible fault (i.e., $B$ is unobserved) can lead to a different MAP explanation than the {\em observation} that $B$ takes its default value; e.g., if $\mprob{b} = 0.6$, $\cprob{a}{c,b} = 0.6$, and $\cprob{a}{c,\bar b} = 0.3$ we have that $\cprob{a}{c} = 0.48$ yet as $\cprob{a}{c,b} = 0.6$ so the MAP explanation changes from $\bar a$ to $a$ on the observation of the {\em default} value $b$. This information is lost in the marginalization step but it helps motivate {\em why} $\bar a$ is the best explanation of $c$.

\begin{figure}[h!]
    \centering
    \begin{subfigure}[b]{0.45\textwidth}
        \centering
        \includegraphics[width=5cm]{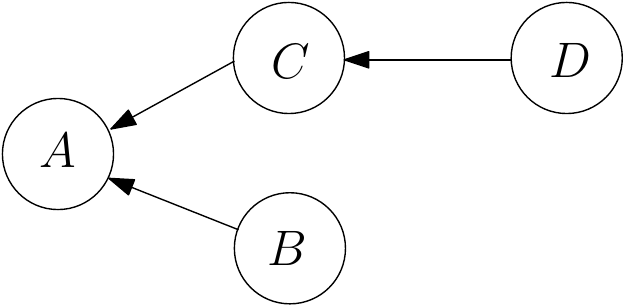}
        \caption{An example small network. Note that the role of $B$ in motivating the best explanation for $A$ given an observation for $C$ is context-dependent and may be different for different observations for $C$, as well as for different conditional probability distributions; hence it cannot be read off the graph alone.} 
        \label{fig:xai_example}
    \end{subfigure}
    \hfill
    \begin{subfigure}[b]{0.45\textwidth}
        \centering
        \includegraphics[width=5cm]{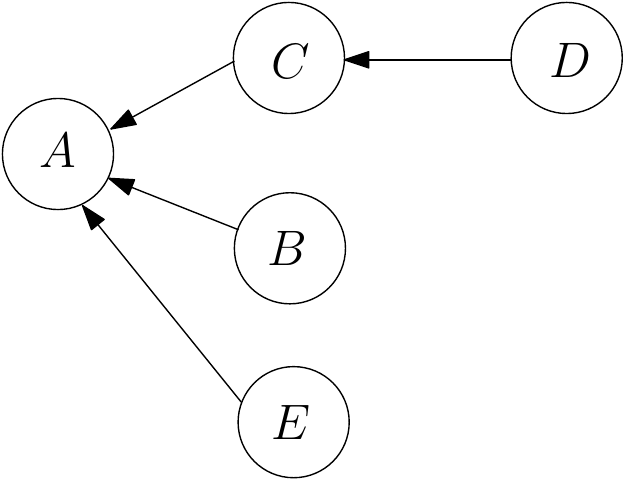}
        \caption{In this adapted network there is not only a potential context-dependent {\em individual} role for $B$ and $E$ in motivating the best explanation for $A$, but there now may also be an {\em interaction effect}: both $B$ and $E$ may each be individually MAP-independent relative to $A$, while the {\em set} $\{B, E\}$ may no longer be MAP-independent relative to $A$. In the latter situation we call $A$ {\em Weakly} MAP-independent from $\{B, E\}$ and we reserve the term {\em Strongly} MAP-independent if the best explanation to $A$ is identical for all joint value assignments to the $\{B, E\}$.}
        \label{fig:inter_example}
     \end{subfigure}
\caption{MAP-independence and Strong vs. Weak MAP-independence}
\label{fig:xai_inter}
\end{figure}

Thus, the relevance of $B$ for the explanation of $A$ may need a {\em different} (and broader) notion of independence, as also suggested by \cite{Kwisthout15}. Indeed, in this example, variable $D$ is irrelevant for explaining $A$ as $D$ is conditionally independent from $A$ given $C$; yet, we could also argue that $B$ is irrelevant for explaining $A$ {\em if its value, were it observed, could not influence the explanation for $A$}. We introduce the term {\em MAP-independence} for this (uni-directional) relationship; we say that $A$ is MAP-independent from $B$ given $C = c$ when $\forall_{b \in \Omega(B)}\mathrm{argmax}_{a^*} \cprob{A=a^*,B = b}{C = c} = a$ for a specific value assignment $a\in\Omega(A)$. In general, for an unobserved variable $I$ we have that the MAP explanation $\mathbf{h^*}$ is MAP-independent from $I$ given the evidence if the explanation would not have been different had $I$ been observed to any value from its domain, and MAP-dependent if this is not the case.

In Figure \ref{fig:xai_inter}\subref{fig:inter_example} we adapted the example network with an additional variable $E$, with the following (relevant) prior and conditional probabilities: $\mprob{b}=0.6$, $\mprob{e}=0.4$, $\cprob{a}{c,b,e}=0.6$, $\cprob{a}{c,\bar b,e}=0.2$,  $\cprob{a}{c,b,\bar e}=0.3$, and $\cprob{a}{c,\bar b,\bar e}=0.3$. The reader can infer that, with observation $C=c$, $A$ is MAP-independent from each of the variables $B$ and $E$ independently, as $\cprob{a}{c,b}=0.42$, $\cprob{a}{c,\bar b}=0.26$ and $\cprob{a}{c,e}=0.44$, $\cprob{a}{c,\bar e}=0.3$, yet the interaction between $E$ and $B$ makes $A$ {\em not} independent from the {\em set} $\{B,E\}$ as the MAP changes from $\bar a$ to $a$ for $B=b,E=e$. We call $A$ {\em weakly} MAP-independent from the set $\{B,E\}$. Note that for singleton variable sets, weak and strong MAP-independence collapse.

An explication of how $I$ may impact or fail to impact the most probable explanation of the evidence will both help motivate the system's advice as well as offer guidance in further decisions, e.g., to decide when to gather additional evidence \cite{Gaag93,Gaag11} to make the MAP explanation more robust. Typically, the number of intermediate variables that impact decisions (such that the explanation is weakly MAP-independent from these variables) is low \cite{Druzdzel94}; however, when interactions between intermediate variables are taken into account, a challenge may be to find the smallest subset of intermediate variables from which the explanation is MAP-independent in the stronger sense. We will further extend upon this optimization problem in the next section. 

Finally, we observe that there is a clear relation between the concept of MAP-independence and the so-called {\em Same-decision probability} defined by Choi and colleagues. Given a threshold decision made while marginalizing out unobserved variables $\mathbf{H}$, this same-decision probability is [...] {\em the probability that we would have made the same decision had we known the states of [...] $\mathbf{H}$} \cite[p1417]{Choi12}. Specifically, for binary decisions and a threshold of $\sfrac{1}{2}$, this probability is $1$ for $\mathbf{H}$ if and only if the decision is MAP-independent of $\mathbf{H}$.

\section{Formal problem definition and results}
\label{sec:Formal}

The computational problem of interest is to decide upon the set $\mathbf{I^+}$, the relevant variables that contribute to establishing the best explanation given the evidence, in the sense as discussed above. We start with giving formal definitions for deciding whether weak or strong MAP-independence of a given set of variables relative to the evidence can be established and the optimization variant of this problem (i.e., to establish the smallest set of relevant variables). The latter problem is only interesting for strong MAP independence; in the weak case this is a trivial iteration over the variables in $\mathbf{I}$. In subsequent sub-sections we elaborate on the computational complexity of these problems as well as consider fixed-parameter tractability results that indicate under which circumstances the problems are feasible.

We start with defining \mindep. For later simplification of the complexity proofs we look at the marginal definition of MAP; note that $\mathrm{argmax}_{\mathbf{H}} \cprob{\mathbf{H}}{\mathbf{e}} = \mathrm{argmax}_{\mathbf{H}} \mprob{\mathbf{H},\mathbf{e}}$.

\dproblem{\mindep}
{A Bayesian network $\bn = (\mathbf{G}_{\mathcal{B}}, \pr)$, where $\mathbf{V(G)}$ is partitioned into a set of evidence nodes $\mathbf{E}$ with a joint value assignment $\mathbf{e}$, a non-empty explanation set $\mathbf{H}$, a non-empty set of nodes $\mathbf{R}$ for which we want to decide MAP-independence relative to $\mathbf{H}$, and a set of intermediate nodes $\mathbf{I}$}
{Is $\forall_{\mathbf{r},\mathbf{r'} \in \Omega(\mathbf{R})} \mathrm{argmax}_{\mathbf{H}} \mprob{\mathbf{H},\mathbf{R} = \mathbf{r}, \mathbf{E} = \mathbf{e}} = \mathrm{argmax}_{\mathbf{H}} \mprob{\mathbf{H},\mathbf{R} = \mathbf{r'}, \mathbf{E} = \mathbf{e}}$}

\noindent Observe that the complement problem \mdep\ is defined similarly with \yes- and \no-answers reversed.

In this formal definition we allow for interaction in the set $\mathbf{R}$. The related problem \wmindep, where we consider MAP-independence for each singleton variable on its own, is defined as follows:

\dproblem{\wmindep}
{As in \mindep}
{Is $\forall_{R \in \mathbf{R}}\forall_{r,r' \in \Omega(R)} \mathrm{argmax}_{\mathbf{H}} \mprob{\mathbf{H},R = r, \mathbf{E} = \mathbf{e}} = \mathrm{argmax}_{\mathbf{H}} \mprob{\mathbf{H},R = r', \mathbf{E} = \mathbf{e}}$}

\noindent Observe, that while $|\Omega(\mathbf{R})| = \mathcal{O}(c^{|\mathbf{R}|})$, where $c = \max_{W \in V(G)}\Omega(W)$, we have that in deciding \wmindep\ we test at most $\mathcal{O}(c|\mathbf{R}|)$ assignments.

\noindent Apart from assessing strong MAP-independence for specific sets $\mathbf{R}$, we might be interested in the maximal subset $\mathbf{R}\subseteq\mathbf{I}$ such that $\mathbf{H}$ is strongly MAP-independent from $\mathbf{R}$ given the evidence. We refer to this problem as \maxmindep, defined (as a decision problem) as follows:

\dproblem{\maxmindep}
{A Bayesian network $\bn = (\mathbf{G}_{\mathcal{B}}, \pr)$, where $\mathbf{V(G)}$ is partitioned into a set of evidence nodes $\mathbf{E}$ with a joint value assignment $\mathbf{e}$, a non-empty explanation set $\mathbf{H}$, a set of potentially relevant intermediate nodes $\mathbf{I^P}$, and a set of other intermediate nodes $\mathbf{I^O}$; furthermore, a positive integer $k$}
{Is there a subset $\mathbf{R}\subseteq\mathbf{I^P}$ of size at least $k$, such that $\forall_{\mathbf{r},\mathbf{r'} \in \Omega(\mathbf{R})} \mathrm{argmax}_{\mathbf{H}} \cprob{\mathbf{H},\mathbf{R} = \mathbf{r}}{\mathbf{e}} = \mathrm{argmax}_{\mathbf{H}} \cprob{\mathbf{H},\mathbf{R} = \mathbf{r'}}{\mathbf{e}}$}

\noindent Observe that we allow in this problem definition for variables that are not part of the relevance assessment. These variables could, for example, be variables that are conditionally independent from the explanation given the evidence and therefore not relevant to be part of the formal assessment.

\subsection{Computational complexity}

We will show in this sub-section that an appropriate decision variant of \mindep\ is \conppp-complete and thus resides in the same complexity class as the \monotone\ problem \cite{Gaag04}. For \wmindep\ it is easy to show membership of \conppp, by a similar argument, but hardness is unlikely given that the universal quantifier runs over a linear number of cases. For \maxmindep\ we show that the problem is in \npnppp\ and is polynomial-time many-one reducible from the \partamajsat\ problem which we conjecture (but were unable to prove) to be \npnppp-complete. We first proceed with proving complexity of \mindep. 

In order to establish the computational complexity of these problems, there are several caveats in the problem definitions above. The first one is that we wish to isolate the complexity of verifying that the MAP stays the same for all variables in the relevance assessment, and separate it from establishing MAP in the first place (an \nppp-hard problem). The second caveat is that, for complexity analyses, inference (and MAP) are defined as decision problems, rather than functional problems; the above definitions include functional aspects (such as maximization problems) that don't fit well with traditional complexity analyses of comparable problems.

To address the first caveat we require $\mathbf{R}$ to be non-empty and to include a joint value assignment $\mathbf{h^*}$ to $\mathbf{H}$ in the input. We do {\em not} explicitly require in the formal definition of the decision problem that $\mathbf{h^*}$ is in fact the MAP explanation, as this would effectively make \mindep\ a so-called promise problem \cite{Goldreich06}, obfuscating the complexity analysis. 

Prior to considering the second caveat let us first look at a common definition\cite{Kwisthout11} of the decision variant \map-D:

\dproblem{\map-D}
{A Bayesian network $\bn = (\mathbf{G}_{\mathcal{B}}, \pr)$, where $\mathbf{V(G)}$ is partitioned into a set of evidence nodes $\mathbf{E}$ with a joint value assignment $\mathbf{e}$, a non-empty explanation set $\mathbf{H}$, and a set of intermediate nodes $\mathbf{I}$; furthermore, a rational number $q \in [0,1\rangle$}
{Is there a joint value assignment $\mathbf{h}$ to $\mathbf{H}$ such that $\mprob{\mathbf{H}=\mathbf{h}, \mathbf{E} = \mathbf{e}} > q$}

This captures the crucial sources of complexity (summation over the intermediate variables, maximization over the hypothesis set) of \map. Note that \map-D is \nppp-complete; further observe that both the summation and the maximization aspects of \map\ are here formulated as decisions. Taking this into account we formulate \mindep-D as follows:

\dproblem{\mindep-D}
{A Bayesian network $\bn = (\mathbf{G}_{\mathcal{B}}, \pr)$, where $\mathbf{V(G)}$ is partitioned into a set of evidence nodes $\mathbf{E}$ with a joint value assignment $\mathbf{e}$, a non-empty explanation set $\mathbf{H}$ with a joint value assignment $\mathbf{h^*}$, a non-empty set of nodes $\mathbf{R}$ for which we want to decide MAP-independence relative to $\mathbf{H}$, and a set of intermediate nodes $\mathbf{I}$; rational number $s$}
{Is, for each joint value assignment $\mathbf{r}$ to $\mathbf{R}$, $\mprob{\mathbf{h^*},\mathbf{r},\mathbf{e}} > s$}

Note that as a decision variant of \mindep\ there is still a slight caveat, as the probability of $\mprob{\mathbf{h^*},\mathbf{r},\mathbf{e}}$ can be different for each joint value assignment $\mathbf{r}$, implying that the `generic' threshold $s$ can either be too strict ($\mathbf{h}$ is still the MAP explanation although the test fails) or too loose (there is another explanation $\mathbf{h'}$ which is the MAP explanation although the test passes). As the number of joint value assignments $|\Omega(\mathbf{R})|$ can be exponential in the size of the network (and thus we cannot include individual thresholds $s_i$ in the input of the decision problem without blowing up the input size), this nonetheless appears to be the closest decision problem variant that still captures the crucial aspects\footnote{One of the reviewers of the conference version of this paper \cite{Kwisthout21} asked whether a `suitable' (neither too loose nor too strict) threshold $s$ can (or is guaranteed) to exist. A trivial example with two unconnected binary nodes $H$ (the MAP variable, with $\mprob{H=h}=0.51$ and $h^* = h$) and uniformly distributed $R$ (and either no evidence or another unconnected and arbitrarily instantiated evidence node), and a threshold $s=0.5$ suffices to show that indeed there exists networks where a suitable $s$ exists. On the other hand it is also easy to construct a small network where no suitable $s$ exists; take a ternary node $H$ instead, with binary and uniformly distributed $R$ as parent (and again no or unconnected evidence) and define $\cprob{H=h_1}{r}=0.56$, $\cprob{H=h_2}{r}=0.44$, $\cprob{H=h_3}{r}=0$, $\cprob{H=h_1}{\bar r}=0.44$, $\cprob{H=h_2}{\bar r}=0.4$, and $\cprob{H=h_3}{\bar r}=0.16$; furthermore let $h* = h_1$. Note that the joint probabilities $\mprob{H,R} = \cprob{H}{R} \times 0.5$ as $R$ is uniformly distributed. While $h_1$ is the MAP for each value of $R$, a threshold of $0.22$ is too loose while a threshold of $0.22+\epsilon$ is too strict for each $\epsilon > 0$.}  of \mindep, particularly the universal quantification over all joint value assignments to $\mathbf{R}$ and the probabilistic inference and threshold test.

For the hardness proof we reduce from the canonical satisfiability \conppp-complete variant \amajsat\ \cite{Wagner86} defined as follows:

\dproblem{\amajsat}
{A Boolean formula $\phi$ with $n$ variables $\{x_1,\ldots,x_n\}$, partitioned into the sets $\mathbf{A} = \{x_1,\ldots,x_k\}$ and $\mathbf{M} = \{x_{k+1},\ldots,x_n\}$ for some $1 \leq k \leq n$}
{Does, for every truth instantiation $\mathbf{x_a}$ to $\mathbf{A}$, the majority of truth instantiations $\mathbf{x_m}$ to $\mathbf{M}$ satisfy $\phi$}

As a running example for our reduction we use the formula $\phi_{ex} = \neg(x_1 \wedge x_2)\vee (x_3 \vee x_4)$, with $\mathbf{A} = \{x_1,x_2\}$ and $\mathbf{M} = \{x_3,x_4\}$. This is a \yes-instance of \amajsat: for each truth assignment to $\mathbf{A}$, at least three out of four truth assignments to $\mathbf{M}$ satisfy $\phi$. 

We construct a Bayesian network $\mathcal{B}_{\phi}$ from a given Boolean formula $\phi$ with $n$ variables. For each propositional variable $x_i$ in $\phi$, a binary stochastic variable $X_i$ is added to $\mathcal{B}_{\phi}$, with possible values \true\ and \false\ and a uniform probability distribution. For each logical operator in $\phi$, an additional binary variable in $\mathcal{B}_{\phi}$ is introduced, whose parents are the variables that correspond to the input of the operator, and whose conditional probability table is equal to the truth table of that operator. For example, the value \true\ of a stochastic variable mimicking the \textit{and}-operator would have a conditional probability of $1$ if and only if both its parents have the value \true, and $0$ otherwise. The top-level operator in $\phi$ is denoted as $V_{\phi}$. In Figure \ref{fig:example_construct} the network $\mathcal{B}_{\phi}$ is shown for the formula $\neg(x_1 \wedge x_2)\vee (x_3 \vee x_4)$.

\begin{figure}[ht!]
    \centering
    \includegraphics[width=7cm]{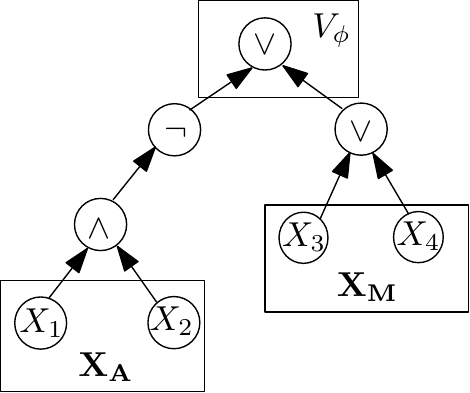}
    \caption{The network $\mathcal{B}_{\phi}$ created from the \amajsat\ instance $(\phi, \{x_1,x_2\},\{x_3,x_4\})$ per the description above.}
    \label{fig:example_construct}
\end{figure}

\begin{theorem}\mindep-D is \conppp-complete.
\end{theorem}
\begin{proof}To prove membership in \conppp, we give a falsification algorithm for \no-answers to \mindep-D instances, given access to an oracle for the \pp-complete \cite{Cooper90} \infer\ problem. Let $(\mathcal{B},\mathbf{E},\mathbf{e},\mathbf{H},\mathbf{h^*},\mathbf{R},\mathbf{I},s)$ be an instance of \mindep-D. We non-deterministically guess a joint value assignment $\mathbf{\bar r}$, and use the \infer\ oracle to verify that $\mprob{\mathbf{h^*},\mathbf{\bar r},\mathbf{e}} \leq s$, which by definition implies that $\mathbf{H}$ is not MAP-independent from $\mathbf{R}$ given $\mathbf{E} = \mathbf{e}$, establishing a falsification.

To prove hardness, we reduce from \amajsat. Let $(\phi, \mathbf{X_A}, \mathbf{X_M})$ be an instance of \amajsat\ and let $\mathcal{B}_{\phi}$ be the Bayesian network created from $\phi$ as per the procedure described above. We set $\mathbf{R} = \mathbf{X_A}$, $\mathbf{H} = \{V_{\phi}\}$, $\mathbf{E}=\varnothing$, $\mathbf{I} = \mathbf{V(G_{\mathcal{B}})}\setminus\mathbf{R}\cup\{V_{\phi}\}$, $\mathbf{h^*} = \{\true\}$, and $s = 2^{-|\mathbf{R}|-1}$. \\

\begin{tabular}{lp{10.5cm}}
$\implies$ & Assume that $(\phi, \mathbf{X_A}, \mathbf{X_M})$ is a \yes-instance of \amajsat, i.e., for every truth assignment to $\mathbf{X_A}$, the majority of truth assignments to $\mathbf{X_M}$ satisfies $\phi$. Then, by the construction of $\mathcal{B}_{\phi}$, we have $\sum_{\mathbf{r}} \mprob{V_{\phi}=\true,\mathbf{r}} > \sfrac{1}{2}$ and so, as the variables in $\mathbf{R}$ are all uniformly distributed, $\mprob{V_{\phi}=\true,\mathbf{r}} > 2^{-|\mathbf{R}|-1}$ for every joint value assignment $\mathbf{r}$ to $\mathbf{R}$, and so this is a \yes-instance of \mindep-D. \\
$\impliedby$ & Assume that $(\mathcal{B},\varnothing,\varnothing,V_{\phi},\true,\mathbf{R},\mathbf{I}, 2^{-|\mathbf{R}|-1})$ is a \yes-instance of \mindep-D. Given the construction this implies that for all joint value assignments $\mathbf{r}$ it holds that $\mprob{V_{\phi}=\true,\mathbf{r}} > 2^{-|\mathbf{R}|-1}$. But this implies that for all truth assignments to $\mathbf{X_A}$, the majority of truth assignments to $\mathbf{X_M}$ satisfies $\phi$, hence, this is a \yes-instance of \amajsat. \\ 
\end{tabular}\\

\noindent Observe that the construction of $\mathcal{B}_{\phi}$ takes time, polynomial in the size of $\phi$, which concludes our proof. Furthermore, the result holds in the absence of evidence.
\end{proof}

\begin{corollary}
\mdep-D is \nppp-complete.
\end{corollary}

\noindent Further observe that the \conppp-membership proof of \mindep-D trivially extends to \wmindep-D by giving a suitable counterexample; however, this problem is unlikely to be complete for this class. We now turn to the complexity of \maxmindep-D. We will show that the following problem can be reduced in polynomial time to \maxmindep-D:

\dproblem{\partamajsat}
{A Boolean formula $\phi$ with $n$ variables $\{x_1,\ldots,x_n\}$, positive integer $k < n$}
{Is there a partition of $\{x_1,\ldots,x_n\}$ into the sets $\mathbf{A}$ and $\mathbf{M}$, with $|\mathbf{A}| \geq k$, such that for every truth instantiation $\mathbf{x_a}$ to $\mathbf{A}$, the majority of truth instantiations $\mathbf{x_m}$ to $\mathbf{M}$ satisfies (with $\mathbf{x_a}$) $\phi$}

It can be shown that \partamajsat\ is in \npnppp: non-deterministically guess the partition and use an oracle for \amajsat\ to verify the desired property. However, despite substantial effort, we did not succeed in establishing hardness; while we can give a reduction from \partamajsat\ to \maxmindep, implying that \maxmindep\ is as least as hard as \partamajsat, we therefore cannot show \npnppp-hardness of \maxmindep\footnote{The canonical complete problems for classes in the counting hierarchy can be traced back to Wagner's \cite{Wagner86} extension of early work by Meyer and Stockmeyer \cite{Meyer72,Stockmeyer73} which explicitly uses the given partitioning of the variables and oracle machines in their proof. The flexibility in the size of the sub-sets in the partition makes it difficult to force translation of a traditional \eamajsat\ (the canonical complete satisfiability variant for \npnppp) instance into a \partamajsat\ instance.}.

We start by defining the decision problem in a similar manner as for \mindep-D.

\dproblem{\maxmindep-D}
{A Bayesian network $\bn = (\mathbf{G}_{\mathcal{B}}, \pr)$, where $\mathbf{V(G)}$ is partitioned into a set of evidence nodes $\mathbf{E}$ with a joint value assignment $\mathbf{e}$, a non-empty explanation set $\mathbf{H}$ with a joint value assignment $\mathbf{h^*}$, a set of potentially relevant intermediate nodes $\mathbf{I^P}$, and a set of other intermediate nodes $\mathbf{I^O}$; rational number $s$, positive integer $k$}
{Is there a subset $\mathbf{R}\subseteq\mathbf{I^P}$ of size at least $k$, such that for each joint value assignment $\mathbf{r}$ to $\mathbf{R}$, $\mprob{\mathbf{h^*},\mathbf{r},\mathbf{e}} > s$}

Observe that \maxmindep-D is in \npnppp, as we can non-deterministically guess $\mathbf{R}$ and invoke an oracle for \mindep-D. We now turn to the reduction.

\begin{theorem}\partamajsat\ $\leq^P_m$ \maxmindep-D.
\end{theorem}

\begin{proof}Let $(\phi,k)$ be an instance of \partamajsat, and let $\mathcal{B}_{\phi}$ be the Bayesian network created from $\phi$ as per the procedure described above, but without partitioning the variables $X_i$ into sets $\mathbf{X_A}$ and $\mathbf{X_M}$; we will refer to the set of variables $X_i$ as $\mathbf{X}$. We set $\mathbf{I^P} = \mathbf{X}$, $\mathbf{H} = \{V_{\phi}\}$, $\mathbf{E}=\varnothing$, $\mathbf{I^O} = \mathbf{V(G_{\mathcal{B}})}\setminus\mathbf{I^P}\cup\{V_{\phi}\}$, $\mathbf{h^*} = \{\true\}$, and $s = 2^{-k-1}$. \\

\begin{tabular}{lp{10.5cm}}
$\implies$ & Assume that $(\phi, k)$ is a \yes-instance of \partamajsat. Then there exists a partition of the variables $X_i$ of $\phi$ in the sets $\mathbf{X_A}$ and $\mathbf{X_M}$, with $|\mathbf{X_A}| \geq k$, such that for every truth assignment to $\mathbf{X_A}$, the majority of truth assignments to $\mathbf{X_M}$ satisfies $\phi$. Observe that if this property holds for $\mathbf{X_A}$ with $|\mathbf{X_A}| > k$, it also holds for every subset $\mathbf{X_A^k} \subset\mathbf{X_A}$ with $|\mathbf{X_A^k}| = k$, hence, we assume $|\mathbf{X_A}| = k$ for ease of exposition. We set $\mathbf{R} = \mathbf{X_A}$. By the construction of $\mathcal{B}_{\phi}$, we have $\sum_{\mathbf{r}} \mprob{V_{\phi}=\true,\mathbf{r}} > \sfrac{1}{2}$ and so, as the variables in $\mathbf{R}$ are all uniformly distributed, $\mprob{V_{\phi}=\true,\mathbf{r}} > 2^{-|\mathbf{R}|-1} = 2^{-k-1}$ for every joint value assignment $\mathbf{r}$ to $\mathbf{R}$, and so this is a \yes-instance of \maxmindep-D. \\
$\impliedby$ & Assume that $(\mathcal{B},\varnothing,\varnothing,V_{\phi},\true,\mathbf{I^P},\mathbf{I^O}, k, 2^{-k-1})$ is a \yes-instance of \maxmindep-D. Given the construction this implies the existence of a subset $\mathbf{R}\subseteq\mathbf{I^P}$ with $|\mathbf{X_R}| = k$, such that for all joint value assignments $\mathbf{r}$ to $\mathbf{R}$ it holds that $\mprob{V_{\phi}=\true,\mathbf{r}} > 2^{-k-1}$. But this implies the existence of a partition of the variables $X_i$ of $\phi$ in the sets $\mathbf{X_A}$ and $\mathbf{X_M}$, with $|\mathbf{X_A}|=k$, such that for all truth assignments to $\mathbf{X_A}$, the majority of truth assignments to $\mathbf{X_M}$ satisfies $\phi$. Hence, this is a \yes-instance of \partamajsat. \\ 
\end{tabular}\\

\noindent Observe again that the construction of $\mathcal{B}_{\phi}$ takes time, polynomial in the size of $\phi$, which concludes our proof. Furthermore, again the result holds in the absence of evidence.
\end{proof}

\subsection{Algorithm and algorithmic complexity}

To decide whether a MAP explanation $\mathbf{h^*}$ is MAP-independent from a set of variables $\mathbf{R}$ given evidence $\mathbf{e}$, the straightforward algorithm below\footnote{Source code available online: \url{https://gitlab.socsci.ru.nl/j.kwisthout/most-frugal-explanations}.} shows that the run-time of this algorithm is $\mathcal{O}(\Omega(\mathbf{R})) = \mathcal{O}(2^{|\mathbf{R}|})$ times the time needed for each MAP computation. 

\begin{algorithm}[H]
\SetAlgoLined
\KwIn{Bayesian network partitioned in $\mathbf{E}=\mathbf{e}$, $\mathbf{H}=\mathbf{h^*}$, $\mathbf{R}$, and $\mathbf{I}$.}
\KwOut{\yes\ if $\mathbf{H}$ is MAP-independent from $\mathbf{R}$ given $\mathbf{e}$, \no\ if otherwise.}
 \ForEach{$\mathbf{r}\in\Omega(\mathbf{R})$}
 {
 \If{$\mathrm{argmax}_{\mathbf{H}} \cprob{\mathbf{H},\mathbf{R} = \mathbf{r}}{\mathbf{e}} \neq \mathbf{h^*}$}
 {\Return{\no\;}}
 }
 \Return{\yes\;}
 \caption{Straightforward \mindep\ algorithm}
\end{algorithm}

\noindent This implies that the size of the set $\mathbf{R}$ for which we want to establish MAP independence is the crucial source of complexity assuming constraints that allow \map\ to be computed \cite{Kwisthout11}, respectively approximated \cite{Park04,Kwisthout15}, feasibly. Given known results on fixed-parameter tractability of \map, the following fixed-parameter tractability results can be derived:

\begin{corollary}
Let $c = \max_{W \in V(G)}\Omega(W)$ and let \tw\ be the tree-width of \bn. Then $p$-\mindep\ is fixed-parameter tractable for $p=\{|\mathbf{H}|, |\mathbf{R}|, \tw, c \}$ and $p=\{|\mathbf{H}|, |\mathbf{R}|, |\mathbf{I}|, c \}$.
\end{corollary}

\noindent So, in general, for \mindep\ to be tractable we need to constrain the set $\mathbf{R}$ of variables to consider, and make sure that the \map\ queries can be computed tractably by constraining the cardinality of the variables and size of the hypothesis set, and either the tree-width of the network or the size of the set of intermediate variables. 

The following trivial modifications extend this algorithm to the weak case and to the maximization problem:

\begin{algorithm}[H]
\SetAlgoLined
\KwIn{Bayesian network partitioned in $\mathbf{E}=\mathbf{e}$, $\mathbf{H}=\mathbf{h^*}$, $\mathbf{R}$, and $\mathbf{I}$.}
\KwOut{\yes\ if $\mathbf{H}$ is MAP-independent from $\mathbf{R}$ given $\mathbf{e}$, \no\ if otherwise.}
 \ForEach{$R\in\mathbf{R}$}
 {
 \ForEach{$r\in\Omega(R)$}
 {
 \If{$\mathrm{argmax}_{\mathbf{H}} \cprob{\mathbf{H},R = r}{\mathbf{e}} \neq \mathbf{h^*}$}
 {\Return{\no\;}}
 }
 }
 \Return{\yes\;}
 \caption{\wmindep\ algorithm variant}
\end{algorithm}

\begin{algorithm}[H]
\SetAlgoLined
\KwIn{Bayesian network partitioned in $\mathbf{E}=\mathbf{e}$, $\mathbf{H}=\mathbf{h^*}$, $\mathbf{I^P}$, and $\mathbf{I^O}$; integer $k$.}
\KwOut{\yes\ if there is a subset $\mathbf{R}\subseteq\mathbf{I^P}$ of size $|\mathbf{R}| = k$ such that $\mathbf{H}$ is MAP-independent from $\mathbf{R}$ given $\mathbf{e}$, \no\ if otherwise.}
 \ForEach{$\mathbf{R}\subseteq\mathbf{I^P}$ of size $|\mathbf{R}| = k$}
 {
 \ForEach{$\mathbf{r}\in\Omega(\mathbf{R})$}
 {
 \If{$\mathrm{argmax}_{\mathbf{H}} \cprob{\mathbf{H},\mathbf{R} = \mathbf{r}}{\mathbf{e}} \neq \mathbf{h^*}$}
 {\bf break\;}
 }
 \Return{\yes\;}
 }
 {\Return{\no\;}}
 \caption{\maxmindep\ algorithm}
\end{algorithm}

We conclude that $p$-\wmindep\ is fixed-parameter tractable exactly when MAP is fixed-parameter tractable, and that $p$-\maxmindep\ is fixed-parameter tractable for $p=\{|\mathbf{H}|, |\mathbf{I^P}|, \tw, c \}$,  $p=\{|\mathbf{H}|, |\mathbf{I}|, c \}$. Note that the latter results somewhat obfuscate the difference in complexity between \mindep\ and \maxmindep; in the third algorithm the additional \textbf{foreach} loop (looping over all subsets of $\mathbf{I^P}$ of size $k$) collapses if $|\mathbf{I^P}|$ is fixed.

\section{Conclusion and future work}
\label{sec:Conclusion}

In this paper we introduced MAP-independence as a formal notion, relevant for decision support and justification of decisions. In a sense, MAP-independence is a relaxation of conditional independence, suggested by Pearl \cite{Pearl87} to be a scaffold for human context-dependent reasoning. We suggest that MAP-independence may be a useful notion to further explicate the variables that are {\em relevant} for the establishment of a particular MAP explanation. Establishing whether the MAP explanation is MAP-independent from a set of variables given the evidence (and so, whether these variables are relevant for justifying the MAP explanation) is a computationally intractable problem; we expect that the problem of finding the maximal set for which the MAP explanation is MAP-independent is even harder, modulo commonly assumed separation properties of the counting hierarchy. However, for a {\em specific} variable of interest $I$ (or a small set of these variables together) the problem is tractable whenever MAP can be computed tractably; in practice, this may suffice for usability in typical decision support systems.

There are many related problems of interest that one can identify, but which will be delegated to future work. For example, if the set of relevant variables is large, one might be interested in deciding whether observing one variable can bring down this set (by more than one, obviously). Initial work in this area has been proposed by \cite{Janssen22} who looked at several heuristic approaches (e.g., Gini-index \cite{Sent07} and linear-value utility \cite{Gaag93}) to decide which variable may best be observed.

Another related problem would be to decide upon the observations that are relevant for the MAP explanation (i.e., had we not observed $E \in \mathbf{E}$ or had we observed a different value, would that change the MAP explanation?). This would extend previous work \cite{Meekes15} where the relevance of $E$ for computing a posterior probability (conditioned on $\mathbf{E}$) was established.

The notion of MAP-independence can, for practical purposes, be to strict; furthermore, when an explanation is MAP-dependent on a set of variables, the `amount of impact' can be relevant. Various authors have proposed a form of {\em quantification} of MAP-independence, for example, based on the summed probability of joint value assignments to $\mathbf{R}$ for which the MAP explanation does not change \cite{Valero-Leal22}, the proportion of such joint value assignments (ignoring their probability)
\cite{Elteren22}, or the average structural impact (in terms of Hamming distance) on the best explanation \cite{Jacobs22}. A thorough overview of these and possibly other ways of quantifying MAP-independence would be welcome.

Finally, in order to test its practical usage, the formal concept introduced in this paper should be put to the empirical test in an actual decision support system to establish whether the justifications supported by the notion of MAP-independence actually help understand and appreciate the system's advise.

\section*{Acknowledgments}
The author is grateful for the valuable feedback received from Nils Donselaar, and acknowledges the highly constructive and thorough comments from the anonymous reviewers of both this version as well as an earlier conference version of this paper \cite{Kwisthout21}.

%\newpage
\bibliographystyle{splncs04}
\bibliography{Kwisthout2021}

\end{document}